\documentclass{article}

\usepackage{arxiv}

\usepackage[utf8]{inputenc} 
\usepackage[T1]{fontenc}    
\usepackage{hyperref}       
\usepackage{url}            
\usepackage{booktabs}       
\usepackage{amsfonts}       
\usepackage{nicefrac}       
\usepackage{microtype}      
\usepackage{lipsum}		
\usepackage{graphicx}
\usepackage{natbib}
\usepackage{doi}

\title{A template for the \emph{arxiv} style}

\usepackage{graphicx} 

\usepackage{times}
\usepackage{amsmath}
\usepackage{amssymb}
\usepackage{mathtools}
\usepackage{amsthm}
\usepackage{xcolor}

\newcommand{\ruth}[1]{}

\newcommand{\ignoretosca}[1]{}
  \newcommand{\reals}{\mathbb{R}}
   \newcommand{\naturals}{\mathbb{N}}
   
   \renewcommand{\Pr}{\mathbb{P}}

   \newcommand{\indct}[1]{\boldsymbol{1}\!\left[ #1 \right]}

  \newcommand{\F}{{\mathcal F}}
  
  \renewcommand{\H}{{\mathcal H}}
  \newcommand{\C}{{\mathcal C}}

  \newcommand{\X}{\mathcal{X}}
  \newcommand{\Y}{\mathcal{Y}}
  \newcommand{\Z}{\mathcal{Z}}

  \newcommand{\Q}{\mathcal{Q}}

  \newcommand{\Pcal}{{\mathcal P}}

  \newcommand{\opt}{\mathrm{opt}}

   \newcommand{\supp}{\mathrm{supp}}

\newcommand{\learner}{\mathcal{A}}

\newtheorem{claim}{Claim}

\newtheorem{lemma}{Lemma}
\newtheorem{definition}{Definition}
\newtheorem{theorem}{Theorem}
\newtheorem{corollary}{Corollary}

\title{Impossibility of Characterizing Distribution Learning\\
-- a simple solution to a long-standing problem}
\author{anonymous authors}
\author{ {Tosca~Lechner}\\
	Cheriton School of Computer Science\\
	University of Waterloo, Canada\\
 and\\
	Vector Institute, Canada\\
	\texttt{tlechner@uwaterloo.ca} \\
	\And
	{Shai Ben-David}\\
	Cheriton School of Computer Science\\
	University of Waterloo, Canada\\
 and\\
Vector Institute, Canada\\
	\texttt{shai@uwaterloo.ca} \\
}

\renewcommand{\shorttitle}{On the Impossibility of Characterizing Distribution Learning}

\hypersetup{
pdftitle={On the Impossibility of Characterizing Distribution Learning},
pdfsubject={},
pdfauthor={Tosca Lechner, Shai Ben-David},
pdfkeywords={Distribution Learning, Combinatorial Dimension},
}

\begin{document}
\maketitle
\begin{abstract}
We consider the long-standing question of finding a parameter of a class of probability distributions that characterizes its PAC learnability. We provide a rather surprising answer - no such parameter exists.
Our techniques allow us to show similar results for several general notions of characterizing learnability
and for several learning tasks. We show that there is no notion of dimension that characterizes the \emph{sample complexity} of learning distribution classes.
We then consider the weaker requirement of only characterizing learnability (rather than the quantitative sample complexity function). We propose some natural requirements for such a characterization and go on to show that there exists no characterization of learnability that satisfies these requirements for classes of distributions.
Furthermore, we show that our results hold for various other learning problems. In particular, we show that there is no notion of dimension characterizing (or characterization of learnability) for any of the tasks: \emph{classification learning} for distribution classes, learning of binary classifications w.r.t. a restricted set of marginal distributions and learnability of classes of real-valued functions with continuous losses.  
\end{abstract}
\section{Introduction}

The celebrated `fundamental theorem of statistical learning' provides a clean characterization of PAC learnability of binary classification in terms of the combinatorial Vapnik Chervonenkis dimension (VC-dimension) \cite{BEHW89}. The finiteness of the VC-dimension characterizes the learnability of any hypothesis class. Furthermore, the learning rates for any class $H$ of binary valued functions are fully determined (up to constants) by the VC-dimension of that class.

That result sparked a quest for notions of dimension that similarly characterize the learnability of other learning tasks.
For some tasks, such as online learning of binary classifiers 
and learning p-concepts such dimensions have indeed been established. For other statistical learning tasks, some parameters have been proposed but not proven to provide the required characterizations.

In contrast, the results of \cite{BD_undecidable} showed for the first time that for some type of problems, such as EMX learnability, no such characterization can be proved to exist by the common axioms of mathematics (the ZFC set theory).

In this paper, we investigate the existence of characterizing dimensions for several statistical learning problems for which that question remained open (most notably the learnability of classes of discrete probability distributions). We show that, quite surprisingly, no such characterizations exist\footnote{At the end of section \ref{Notions_of_char} we discuss the relationship between our notions of learnability characterizations and the notion of \textit{combinatorial dimension} defined by \cite{BD_undecidable}.}.

Our results answer some long-standing open questions;
The survey paper \cite{Diakonikolas16} asks (Open Problem 1.5.1):
\textit{"Is there a “complexity measure” of a distribution
class C that characterizes the sample complexity of learning C?"}

 \cite{Hopkins_et_al}
 state \textit{"Unlike the
standard model, very little is known about distribution-family learnability. While a number of works have
made some progress on this front, a characterization of learnability remains elusive despite
some 30 years of effort".} (end of Section 4 there).

\cite{BenedekI91} ask about the characterization of PAC learnability of binary-valued classifiers w.r.t. a given class of probability distributions. They conjecture a characterization that is refuted by \cite{Dudley94} The latter repeats the question of finding a characterization for that task. Similar open questions are later stated by \cite{Kulkarni_at_al} and \cite{vidyasagar_et_al}.

\subsection{Notions of characterization of learning tasks}
Towards showing the ``characterization" of some learning tasks is impossible, we need clear definitions of what such characterizations are.

We consider two common types of characterization of learning:
\begin{enumerate}
    \item Quantitative notions that reflect the sample complexity of the learning task (the way the fundamental theorem of statistical learning shows that the \textit{Vapnik-Chervonenkis dimension} characterizes the learning rates of leaning w.r.t. a given hypothesis class).
    \item Qualitative notions that distinguish between learnable and non-learnable classes of models.
\end{enumerate}

We provide formal requirements for both types of characterizations (Section \ref{Notions_of_char}).

All of the characterizations of statistical learning tasks that we are aware of (including, VC-dimension, Littlestone dimension, fat-shattering dimension, Natarajan dimension, Graph dimension etc.) satisfy those requirements characterizations, and so do all of the notions conjectured to characterize some of the tasks for which no characterization had been proved. 

Our main results show that no such characterization is possible for a variety of learning tasks, including the task of learning discrete distributions(\cite{kearns1994paclearningdistributions,devroyelugosibook2001,DBLP:books/sp/Silverman86}),
\emph{classification learning} for distribution classes, learning of binary classifications w.r.t. a restricted sets of marginal distributions(\cite{BenedekI91}), and learnability of classes of real-valued functions with continuous losses. 





\subsection{Paper Outline}
In Section~\ref{sec:setup} we give a general definition for the kind of statistical learning models we will consider (for which distribution learning is a special case) and review some general definitions of ordered sets which we will use later in the paper. In Section~\ref{Notions_of_char} we introduce our notions of characterization of learnability. In Subsection~\ref{sec:samplecomplexity} we introduce quantitative notions of characterization for learnability, which aim to characterize the sample complexity. We state a combinatorial condition of learning tasks that implies that no such characterization exists.
In Section~\ref{sec:finchar} we define a qualitative notion characterization of learnability (Definition~\ref{def:finchar}), which is only required to distinguish learnable from non-learnable classes. 
We then show that there are some general conditions which imply uncharacterizability of a learning task (Theorem~\ref{thm:finchar}). In Section~\ref{sec:distribtionlearning}, we use the theorems of Section~\ref{Notions_of_char} to show the impossibility of characterizing distribution learning, for both quantitative (Theorem~\ref{thm:distributionweakdim}) and qualitative (Theorem~\ref{thm:finchar}) notions of characterization. 
We also show an impossibility result for characterizing classes of distribution which are learnable with polynomial sample complexity for a slightly more restrictive notion of qualitative characterization (Theorem~\ref{thm:polydistrfinchar}). 
Section~\ref{sec:others} shows uncharacterizability for other learning tasks using the results from Section~\ref{Notions_of_char} and following the same construction ideas as in Section~\ref{sec:distribtionlearning}. In particular, we show impossibility of quantitative and qualitative characterizations of classification learning of distribution classes (Theorem~\ref{thm:classweakdimension} and Theorem~\ref{thm:classfinchar}) and learning of real-valued functions with continuous losses (Theorem~\ref{thm:contlossweakdim} and Theorem~\ref{thm:contlossfinchar}). Lastly, we discuss some implications of our results and perspectives for future research in Section~\ref{sec:discussion}.

\section{Setup}\label{sec:setup}




\subsection{Learning model}
We consider a general notion of learning tasks. These consist of the following elements
\begin{itemize}
    \item a domain $\Z$ from which the input instances/training instances are sampled
    \item A class of benchmark models $H$ (in some cases we denote it by $\Q$).
    \item A class of permissible data generating distributions $\mathcal{P}\subset \Delta(\Z)$, where $\Delta(\Z)$ denotes all distributions over the domain $\Z$.
    \item A set of possible outputs of a learner $\F$. Usually, $H \subseteq \F$.
    \item A loss/ approximation measure $L: \F \times \Delta{\Z} \to \mathbb{R}_0^+$ (where $\mathbb{R}_0^+$ denotes the set of non-negative real numbers).
\end{itemize}
We denote the approximation error as $\opt(H,P)= \inf_{h\in \H} L(h, P) $.
\begin{definition}{PAC learnability}\label{def:generalPAC}
\begin{itemize}
    \item 
A class $H\times\Pcal$ is $\alpha$-agnostic PAC learnable w.r.t. to measure  $L: \F \times \Delta{\Z} \to \mathbb{R}_0^+$, 
if there is a learner $\mathcal{A}: \bigcup_{m\in\naturals} \Z^m \to \F $ and a sample complexity function $m^{\alpha}_H:(0,1)^2\to\naturals$, such that for every $\epsilon,\delta>0$, every $P\in \mathcal{P}$ and every $m\geq m{\alpha}_{H}(\epsilon,\delta)$, we have 
    \[ L(\mathcal{A}(S),P) \leq \alpha \cdot \opt(H,P) + \epsilon\]
    with probability $1-\delta$ over $S\sim P^m$.
    \item With a slight abuse of notation, we denote by  $m_{H}(\epsilon,\delta)$ the minimum number $m$ that satisfies the above requirement for $\alpha = 1$.
    \item We say a model $\H$ is learnable if $\H \times \Delta(\Z)$ is learnable.
    \item  We say $\H$ is PAC-learnable in the $\emph{realizable}$ w.r.t. $L$ if $\H \times \{P \in \Delta(\Z): opt(H,P) =0 \}$ is PAC learnable with respect to $L$. We will sometimes refer to the sample complexity of realizable learning by $m^{rlzb}$ to distinguish it from the sample complexity of agnostic learning.
    \item We say a class of distributions $\Pcal$ is PAC-learnable with respect to $L: \F \times \Delta(\Z) \to \mathbb{R}_0^+$ if $\F\times \Pcal$ is PAC learnable with respect to $L$.
    
    \end{itemize}
\end{definition}
\begin{definition}
   For a given learning task, we say that a class of outputs $H' \subset \F$ is an $\epsilon$-approximation for $H\times \Pcal$ w.r.t. to $L$, if for every $(h,p)\in \H\times \Pcal$, there is a $h' \in H'$ such that $L(h',p) \leq L(h,p) +\epsilon$.
   Using the same slight abuse of notation as in the definition of PAC learning, we will say that:
   \begin{itemize}
   \item $H'$ is an $\epsilon$-approximation for $H$ if it is an $\epsilon$-approximation for $H\times \Delta(\Z)$ 
   \item $H'$ is an $\epsilon$-approximation for $\Pcal$ if it is an $\epsilon$-approximation for $\F\times \Pcal$ 
   \end{itemize}
\end{definition}









\subsection{Some notions of ordered sets}
\begin{definition}[Cofinality]

Let $(X, \leq)$ be an ordered set. 
\begin{itemize}
\item For subsets $A, B \subseteq X$, we say that $A$ is \textit{cofinal} in $B$ if for every $b \in B$ there exists some $a \in A$ such that $b \leq a$. 
\item The \textit{cofinality of} an ordered set $(X, \leq)$ is the minimal cardinality of a subset $A$ that is cofinal in $X$.
\end{itemize}
    
\end{definition}

\noindent Note that for subsets $A, B, C$ of $X$, if $A$ is cofinal in $B$ and $B$
is cofinal in $C$ then $A$ is cofinal in $C$.

\begin{definition}[Dominance ordering of functions]
Let $(X, \leq_X)$, $(Y, \leq_Y)$ be linearly ordered sets where $X$ has no maximal element. For 
functions $f,g: X \to Y$,
we say that $f$ \textit{eventually dominates} $g$ if there exists some $x \in X $ such that for every $x' \in X$, if $x \leq_X x'$ then $g(x') \leq_Y f(x')$.
We denote this relation by $g \leq_{ed} f$.

\end{definition}

\begin{claim} Consider $\naturals^{\naturals}$ (the set of all functions from the natural numbers to natural numbers). The cofinality of $(\naturals^{\naturals}, \leq_{ed})$ is uncountable. 
\end{claim}

\begin{proof}
  Consider any countable $A \subseteq \naturals^{\naturals}$, let $\{g_n : n \in \naturals\}$ be an enumeration of the members of $A$. Define $f : \naturals \to \naturals$
by $f(n) = \max \{g_i(n) : i \leq n\}+1$. Clearly $f$ dominates every member of $A$ (and no member of $A$ dominates $f$) showing that $A$ is not cofinal in $(\naturals^{\naturals}, \leq_{ed})$.  
\end{proof}

\section{Notions of characterization of learning tasks} \label{Notions_of_char}
As discussed in the introduction, we consider two types of characterizations - a quantitative notion that characterizes learning rates of classes and a qualitative one that only distinguishes between learnable and non-learnable classes.
we now elaborate our definitions of such notions.

\subsection{Notions of quantitative characterization of statistical learning}
\label{sec:samplecomplexity}
\begin{definition}A \emph{strong sample complexity dimension} is a mapping from $d: \mathcal{C} \to \naturals \cup \{\infty\}$, such that a class $\Q$ of models is PAC learnable if and only if $d(\Q) \neq \infty$ and there are functions $f: \naturals \to \naturals $ and $g:(0,1)^2 \to \naturals$ such that 
for every PAC learnable class of distributions $\Q$, $m_{\mathcal{Q}}(\epsilon, \delta) \leq f(d(\Q))g(\epsilon, \delta)$ for all $(\epsilon, \delta) \in (0,1)$. In other words, there is a sample complexity upper bound function that factorizes into a factor depending only on the dimension of a class and a factor depending only on the accuracy and confidence parameters.
h\end{definition}

Note that the fundamental theorem of statistical learning \cite{BEHW89} shows that the \textit{VC-dimension} is a strong sample complexity dimension for binary classification.


\begin{definition}
A \emph{weak sample complexity dimension} is a mapping from $d: \mathcal{C} \to \naturals \cup \{\infty\}$, such that a class $\Q$ of models is PAC learnable if and only if $d(\Q) \neq \infty$ and there are functions $f: \naturals \times (0,1)^2 \to \naturals $  such that 
for every PAC learnable class of distributions $\Q$, and every $(\epsilon, \delta) \in (0,1)$, $m_{\mathcal{Q}}(\epsilon, \delta) \leq f(d(\Q),\epsilon, \delta)$. In other words, all the information needed about a class of distributions $\Q$ to determine (or upper bound) its sample complexity function $m_{\mathcal{Q}}(\epsilon, \delta)$ is captured in its dimension $d(\Q)$.
\end{definition}

Clearly, every strong sample complexity dimension is also a weak one. We also note that, for a satisfying characterization, one might also require a lower bound of the sample complexity in terms of $f$ and $d$. However, as we are only presenting negative results, it suffices to show that even this less ambitious goal is not achievable.


Our main tool for showing the impossibility of having sample complexity dimensions that provide a quantitative characterization for learning tasks is the following lemma:

   \begin{lemma}\label{lemma:cofinal}
    Let $\mathcal{C}$ denote the family of all learnable classes w.r.t. some learning task. For any given learnable class $\Q$ consider the function $m_{\Q}(\frac{1}{\cdot},\frac{1}{7}): \naturals \to \naturals$ that maps a natural number $k$ to $m_{\Q}(\frac{1}{k},\frac{1}{7}) $. \\
    
    If the set $\{m_{\Q}(\frac{1}{k},\frac{1}{7}): k \in \naturals, \Q\in \mathcal{C} \}$ is cofinal in $\naturals^{\naturals}$ (under the eventual dominance ordering)
    then there exists no weak sample complexity dimension for that task.
    \end{lemma}

\begin{proof}
    Let $d: \mathcal{C} \to \naturals$ be a weak sample complexity dimension. 
    For any $d \in \naturals$, let $ f_d'(k)= f(d,1/k, 1/7)$ for all $k$.
    $\{f_{d}': \Q \in \mathcal{C} \}$ is cofinal in the set of sample complexity functions 
     $ \{ m_{\Q}(1/\cdot,1/7): \naturals\to \naturals: \Q \in \mathcal{C}\}$ (under the $\leq_{ed}$ ordering of functions). Thus the cofinality of $( \{ m_{\Q}(1/\cdot,1/7): \naturals\to \naturals: \Q \in \mathcal{C}\},\leq_{ed})$ is at most countable.
     Since we assume that the set $\{m_{\Q}(\frac{1}{k},\frac{1}{7}): k \in \naturals, \Q\in \mathcal{C} \}$ is cofinal in $\naturals^{\naturals}$, we get a contradiction to  the uncountable cofinality of $\naturals^{\naturals}.$

\end{proof}

\begin{definition} \label{monotone_char}
    For an ordered set $(\X, \leq)$, we say that a notion of dimension $d: \mathcal{C} \to \X$ is \textit{monotonic} if for every pair of classes $\Q_1, \Q_2$,  the implied sample complexity functions $f(d(\Q),1/k, 1/7)$ are monotonically increasing. Namely, 
\[d(\Q_1) \geq d(\Q_2) ~\mbox{implies that} ~ f(d(\Q_1),1/k, 1/7) \geq_{ed} f(d(\Q_2),1/k, 1/7).\]
\end{definition}

\begin{theorem}\label{thm:cofinalreal}
If the set $\{m_{\Q}(\frac{1}{\cdot},\frac{1}{7}): \Q\in \mathcal{C} \}$ is cofinal in $\naturals^{\naturals}$ (under the eventual dominance ordering)
    then
there exists no monotonic real-valued function that is a weak sample complexity dimension for that task.\\
\end{theorem}

\begin{proof}
    Noting that the real numbers have countable cofinality, the proof of the natural-valued dimension applies to the monotonic real-valued dimension as well. 
\end{proof}

\subsection{Notions of qualitative characterization of statistical learning} \label{sec:finchar}
The notions of dimension that we have discussed above were \emph{quantitative} - aiming to capture the sample complexity functions of learning classes.
We showed that such dimensions do not exist for problems like distribution learning.

Lacking quantitative notions one can still seek \emph{qualitative} characterizations of learnability. Namely, conditions that distinguish learnable classes from non-learnable ones.
In the case of binary classification tasks, the distinction between finite and infinite VC-dimension serves as such a characterization.

\begin{definition}\label{def:finchar}
    A \textit{finitary characterization} of learnability for a learning task is a countable set of formulas\footnote{formally, these are properties of the learning problem expressed as first order formulas in many-sorted logic that has types/sorts for elements of the class of models $\F$, generating distributions (members of $\Pcal$), domain elements and rational numbers for values of the loss function $L$. For brevity we keep it be clarified by the examples below.} $W$ such that:
    \begin{enumerate}
        \item A class $H$ is not learnable if and only if it satisfies all the formulas in $W$.
        \item For every $ \alpha \in W$ and every $H$ that satisfies $\alpha$ there is a finite subset $H_{\alpha} \subseteq H$ such that for every $H'$, if $H_{\alpha} \subseteq H'$ then $H'$ satisfies $\alpha$.\\

        We say that a finitary characterization $W$ is \textit{uniformly bounded} if for every $\alpha \in W$ there is a finite number $n_{\alpha}$ such that for every $H$ satisfying $\alpha$ there is a subset $H_{\alpha} \subseteq H$ as above of size at most $n_{\alpha}.$

    \end{enumerate}
    
\end{definition}

    Note that most (if not all) of the known characterizations of learnability of learning tasks are  finitary. 

    \noindent {\bf Examples:}
    \begin{enumerate}
    \item The characterization of binary classification learning by VC-dimension;
    The characterizing $W$ can be any set that contains the formulas stating \textit{"$H$ shatters a set of size $d$"} for infinitely many $d$'s.
    \item The characterization of online learnability by the Littlestone dimension;  The characterizing $W$ can be any set that contains the formulas stating \textit{"$H$ L-shatters a tree of depth $d$"} for infinitely many $d$'s.
 dimension characterizing robust learning    \item The characterization of multi-class learnability by the finiteness of the Natarajan dimension, of by the finiteness of the graph dimension.
    \item The characterizations by a combinatorial dimension based on the one-inclusion graph. In these characterizations,
    the complexity of a problem is demonstrated by a finite graph (the 1-inclusion graph). However, there is no a-priory bound to the sizes of these graphs. The complexity parameter of a given problem is reflected the out-degree of some orientation of that graph. Such characterizations were shown for multi-class learning \cite{DanielyS14}, \cite{BrukhimCDMY22}, and of robust learning \cite{Montasser2022AdversarialCharacterization}.
    These are the only examples we are aware of of a finitary characterization of learnability that are not uniformly bounded. 
    \end{enumerate}

    There are also several conjectured characterizations that fall into this category. For example
    \begin{enumerate}
        \item \cite{BenedekI91} conjecture that a class of binary-valued functions, ${\cal C}$, over some domain set $\X$  
        is learnable w.r.t. a class ${\cal P}$ of probability measures over the same domain  if and only if 
        for every $\epsilon>0$, $\sup_{P \in {\cal P}}N(\epsilon, {\cal C}, P) < \infty$, where $N(\epsilon, {\cal C}, P)$
        is the size of a minimal set $A \subseteq {\cal C}$ such that for every $h \in {\cal C}$, there exists some $h' \in A$ for which $P(h \Delta h') \leq \epsilon$.\\

        To see that this conjectured characterization is finitary, note that for every $P$ and every $\epsilon$, $N(\epsilon, {\cal C}, P)$ is equal, within a factor of 2, to the maximum size of a set $B \subseteq {\cal C}$ such that for every $h \neq h' \in B$,  $P(h \Delta h') > \epsilon$.

        Let $\alpha_{n, \epsilon}$ state that there exists a probability distribution $P \in {\cal P}$ and hypotheses
        $h_1, \ldots h_n \in {\cal H}$ such that for all $i \neq j \leq n$, $P(h_i \Delta h_j) > \epsilon$. Now let 
        $W= \{\alpha_{n, 1/k} : n, k \in \naturals\}$. 
        \item Characterizing the learnability of a class of probability discrete distributions by the finiteness of the VC-dimension of the Yatracos sets induced by that class\footnote{ For a class $F$ of functions
from $X$ to $\reals$, their Yatracos class is the family of subsets of
$X$ defined as
\[Y(F) := \{\{x \in X : f_1(x) \geq f_2(x)\} ~\mbox{for some} ~f_1, f_2 \in F\}\]}. 
    
    \end{enumerate}
\begin{theorem}  \label{thm:finchar}
    The learnability of any learning task that satisfies the following two properties cannot be characterized by a finitary characterization.
    \begin{enumerate}
        \item Every finite union of learnable classes of hypotheses is learnable.
        \item There exists a learnable class $H_0$ and non-learnable classes $\{H_k : k \in \naturals\}$ such that for every specific learning instance and every $k \in \naturals$, $H_0$ is an $\epsilon_k$ approximation of $H_k$ w.r.t. that learning instance and $\lim_{k \to \infty} \epsilon_k = 0$.
    \end{enumerate}
\end{theorem}

\begin{proof}
Assume, b.w.o.c. that $W$ is a finitary characterization of the learning task.
    Let $W=\{\alpha_k : k \in \naturals\}$ be any enumeration of $W$. For each $H_k$ let ${\hat H_k}$ be a finite subset of $H_k$ be such that every $H \supset {\hat H_k}$ satisfies $\alpha_k$ and let ${\hat H} = \bigcup_{k \in \naturals} {\hat H_k} $. On one hand, since for every $k$, ${\hat H_k} \subseteq {\hat H}$, ${\hat H}$ satisfies every $\alpha_k$ and is therefore not learnable (by the first requirement from a characterizing $W$). Towards a contradiction, let us show that ${\hat H}$ is learnable. This holds because given any $\epsilon>0$, the set $K_{\epsilon}=\{k: \epsilon_k \geq \epsilon/2 \}$ is a finite set. Therefore by our assumptions on the learning task, the class $H_0 \cup \bigcup_{k \in K_{\epsilon}} {\hat H_k}$ is learnable. Therefore, for any given $\delta$ there is a learner $L$ and some $m(\epsilon/2, \delta)$ so that training samples of larger size guarantee $(\epsilon/2, \delta)$ success for $L$ on such samples. Since $H_0$ is an $\epsilon/2$ approximation to each $H_k$ for which $k \not\in K_{\epsilon}$ (w.r.t. the given learning instance)
    so being $\epsilon/2$ off the minimum loss minimizer on $H_0$ implies being within $\epsilon$ of the loss minimizer in ${\hat H}$.
\end{proof}

\subsection{Comparing the different notions of characterization}

The two types of characterizations introduced above are related but none of them implies the other. Sample complexity dimensions do not restrict the format (or syntax) of a characterization - the dimension function $d$ can be \textit{any} function (from classes to reals or natural numbers). In this respect, the notion of finitary characterization is more restrictive - it restricts the format of the characterization. 

On the other hand, finitary characterizations are weaker, in the sense that they do not provide any information about learning rates. They are only required to distinguish learnable from non-learnable classes. 

In many cases, however, there are characterizations that meet both definitions. Every notion of sample complexity dimension where the statements of the form $d(H) \geq k$
have finite size `evidences' (like a set of members of $H$ shattering a domain subset).
In such case the finitary characterization set $W$ is just the set of statements
$\{d(H) \geq k : k \in \naturals\}$.

The notion of \textit{finite character dimension} of \cite{BD_undecidable} has two definitions there. The first one requires that it "\textit{Characterizes learnability}: A class $F$ is PAC learnable in the model if and only if
$D(F)$ is finite.
\textit{Of finite character}: $D(F)$ has a “finite” character in the following sense: for every $d \in \naturals$ and a class $F$, the statement $D(F) \geq d$ can be demonstrated by a finite set
of domain points and a finite set of members of $F$". 
This definition is clearly equivalent (for classes of functions) to our notion of finitary characterization.

The second definition there, requiring that the statements $D(F) \geq d$ can be expressed by certain first-order formulas, is more syntactically restrictive than our definition, but on the other hand, more lenient as it does not require the existence of finite evidence sets.



\section{Impossibility of Characterizing Distribution Learning}\label{sec:distribtionlearning}
For now, we consider learning over the domain $\mathcal{X} = \mathbb{N}$. Thus all subsets of our domain are measurable.
We consider learning of distribution classes with respect to total variation distance, i.e. our distance measure over distributions is given by $TV(p_1, p_2) = \sup_{A\subset \mathcal{X}} |p_1(A) - p_2(A)|$. Concretely, we consider the following PAC learning task.

\begin{definition}[ (realizable) PAC learning of a distribution class \cite{DBLP:books/sp/Silverman86,devroyelugosibook2001}]
    We say that a class $\mathcal{Q}$ of probability distributions over some domain set $\mathcal{X}$ is \textit{PAC learnable} if there exists a function $\learner: \bigcup_{m \in \naturals}\mathcal {X}^m \to \mathcal{P}$
    and a function $m_{\mathcal{Q}}^{rlzb}: (0,1)^2 \to \naturals$ such that for every $\epsilon, \delta \in (0,1)^2$ and every $Q \in \Q$, if $m \geq m_{\Q}^{rlzb}(\epsilon, \delta)$ then
    \[\Pr_{S \sim Q^m} [TV(\learner(S), Q) > \epsilon] \leq \delta.\]
\end{definition}

\begin{definition}[$3$-agnostic PAC learning of distribution class\cite{DBLP:books/sp/Silverman86,devroyelugosibook2001}]
       We say that a class $\mathcal{Q}$ of probability distributions over some domain set $\mathcal{X}$ is \textit{PAC learnable} if there exists a function $\learner: \bigcup_{m \in \naturals}\mathcal {X}^m \to \mathcal{P}$
    and a function $m_{\mathcal{Q}}: (0,1)^2 \to \naturals$ such that for every $\epsilon, \delta \in (0,1)^2$ and every $Q \in \Q$,
    
    if $m \geq m_{\Q}(\epsilon, \delta)$ then

    \[\Pr_{S \sim Q^m} [TV(\learner(S), Q) > 3\cdot\inf_{Q'\in \Q}(TV(Q,Q')) + \epsilon] \leq  \delta.\]
\end{definition}
We note that these definitions are special cases of the PAC learning definition in the Setup Section (Definition~\ref{def:generalPAC}), where $\Z=\X=\naturals$, $\F=\Delta(\naturals)$, $\Q= H$ and $L = TV: \Delta(\naturals) \to \Delta(\naturals) \to \mathbb{R}_{0}^+$.

We now state the two main theorems of this section, showing the impossibility of both quantitative as well as quantitative characterizations of distribution learning.

\begin{theorem}\label{thm:distributionweakdim}
There is no weak sample complexity dimension for distribution learning (neither in the realizable nor in the 3-agnostic case of distribution learning). 
\end{theorem}
\begin{corollary}\label{cor:distributionrealweakdim}
There exist no monotonic real-valued function that is a weak sample complexity dimension for distribution learning.\\
\end{corollary}
Corollary~\ref{cor:distributionrealweakdim} follows directly from Theorem~\ref{thm:distributionweakdim} and Theorem~\ref{thm:cofinalreal}.

\begin{theorem}\label{thm:distributionfinchar}
        The learnability of classes of discrete distributions cannot be characterized by a finitary characterization. This statement holds both for realizable PAC learnability and for 3-agnostic PAC learnability.
\end{theorem}
We note, that while we only consider constructions of discrete distributions in this section, the corresponding results on uncharacterizability for \emph{general distribution learning} follow directly from these results.

We will show these theorems using Lemma~\ref{lemma:cofinal} and Theorem~\ref{thm:finchar} respectively. In order to do so, we need to construct classes of distributions that meet the requirements of these more general results. We will first describe a construction that can be used for both theorems. We then show some properties of this construction which will be needed for both theorems, namely an upper bound (Lemma~\ref{lemma:Qetanlearnable}) and a lower bound (Lemma~\ref{lemma:lowerbound}) on its sample complexity. We will then state the proofs of Theorem~\ref{thm:distributionweakdim} and Theorem~\ref{thm:distributionfinchar}. Lastly, we will end the section with discussing an extension of the qualitative impossibility result to classes of distributions with polynomial sample complexity.

Throughout this section we will also need the fact that finite classes are learnable. We state Theorem 3.4 from \cite{AshtianiBHLMP18}, which is a slight rephrasing of Theorem 6.3 from \cite{devroyelugosibook2001}. 

\begin{theorem}{\cite{AshtianiBHLMP18},\cite{devroyelugosibook2001}} \label{thm:finiteclasses}
    For any finite class of distributions $\Q = \{q_1,\dots,q_m\}$, there exists a deterministic algorithm which $3$-agnostic PAC learns $Q$ with sample complexity $m_{\Q}(4\epsilon,\delta) \leq \frac{\log(3m^2) + \log(1/\delta) }{2\epsilon^2}$.
\end{theorem}

We will now describe our construction. For a natural number $n\in \naturals$ and a (usually small) mixture parameter $\eta\in (0,1)$, we define the finite class,
\[P_{\eta,n} = \{(1-\eta)\delta_{0} + \eta U_A : A\subset \{1,\dots,n\}\},\] 

where $\delta_{0}$ denotes the distribution with all its mass on point $0$ and $U_{A}$ denotes a uniform distribution over the set $A$. Intuitively, this class thus consists of a heavy non-flexible part ($(1-\eta)\delta_0$) and a highly flexible part with low weight ($\eta U_A$). For any distribution $p$ the TV-distance to an element of $P_{\eta,n}$ only depends in a small part on the low-weight component. However, the low-weight flexible part, will make this class hard to learn for small $\epsilon$.
We now take union over these classes $P_{\eta,n}$ for different combinations of $\eta$ and $n$, which allows us to control the behaviour of the sample complexity and fulfill the requirements for both.
For sequences sequence $\bar{\eta}: \mathbb{N} \to [0,1]$ and $\bar{n}:\mathbb{N} \to \mathbb{N}$, we define
\begin{equation}\label{distributionconstruction}
\mathcal{Q}_{\bar{\eta},\bar{n}} = \bigcup_{i=1}^{\infty} P_{\bar{\eta}(i), \bar{n}(i)}.
\end{equation}
We furthermore define $\bar{\eta}^{-1}(\epsilon) = \arg\min\{i\in \naturals: $ for every $ j\geq i, \bar{\eta}(j) \leq \epsilon\}$ and $n_{max}(i) = \max_{j\in \{1,\dots i\}} \bar{n}(j)$.
We will now show that infinite classes of this kind can be learnable, even as $n$ grows to infinity (and thus making the class in some sense "infinitely flexible"), by controlling the mixture parameter $\eta$.

\begin{lemma} \label{lemma:Qetanlearnable}
Let $\Q = \mathcal{Q}_{\bar{\eta},\bar{n}}$ with $ \lim_{i\to \infty} \bar{\eta}(i) = 0 $. Then $\mathcal{Q}$ is 3-agnostic PAC learnable with sample complexity $m_{\Q}(\epsilon,\delta) \leq (128 ( \log(3( \bar{\eta}^{-1}(\frac{\epsilon}{4})n_{max}{\bar{\eta}}^{-1}(\frac{\epsilon}{4}) +1)^n) +\log(\frac{1}{\delta}))  )/(\epsilon^2)$.
\end{lemma}

\begin{proof}
Assume $\lim_{i\to\infty} \bar{\eta}(i) = 0$. Let $\epsilon> 0$. Then for every $\frac{\epsilon}{4} >0$ there is an $N$, such that for every $N'\geq N$, $\bar{\eta}(N') < \frac{\epsilon}{4}$. We can now focus on 3-agnostic learning the finite class $\Q'= \{\delta_0\} \cup (\bigcup_{i=0}^{N} P_{\bar{\eta}(i), \bar{n}(i)})$, as learning as for any
$q\in  \Q_{\bar{\eta}, \bar{n}}$, there is $p\in \mathcal{Q}'$ with  $TV(p,q) < \frac{\epsilon}{4}$. Thus by triangle inequality for any $q' \in \Delta(\Z)$, we get $ \inf_{p\in \Q'}TV(p,q') \leq \inf_{q\in \Q} TV(q,q') +\frac{\epsilon}{4}$. Thus if we have a 3-agnostic PAC learner $\mathcal{A}$ for $\Q'$ with sample complexity $m_{\Q'}(\frac{\epsilon}{4},\delta)$, we can use it as a PAC learner for $\Q$ with sample complexity $m_{\Q}(\epsilon,\delta)$. Now using Theorem~\ref{thm:finiteclasses}, we can conclude that $\Q$ is learnable with sample complexity $m_{Q}(\epsilon,\delta) \leq 128\frac{\log(3 (\bar{\eta}^{-1}(\frac{\epsilon}{4})n_{max}(\bar{\eta}^{-1}(\frac{\epsilon}{4}))) ^2) + \log(1/\delta)}{\epsilon^2}$.

\end{proof}
However, we can also lower bound the sample complexity of these kinds of classes in the following way.

\begin{lemma} \label{lemma:lowerbound}
    For $\mathcal{Q} =P_{\eta,4n}$, we  have $m_\mathcal{Q}(\frac{\eta}{8},\frac{1}{7}) \geq m_\mathcal{Q}^{rlzb}(\frac{\eta}{8},\frac{1}{7}) \geq n$. 
\end{lemma}
The construction and argument follow from a common no-free lunch style argument. For details we refer the reader to the appendix.
We can now use this construction and bounds to prove Theorem~\ref{thm:distributionweakdim}.

\begin{proof}[Proof of Theorem~\ref{thm:distributionweakdim}]
Based on the confinality considerations described above, it suffices to show that the set $ \{ m_{\mathcal{Q}}(1/\cdot, 1/7) : \Q \in m_{\mathcal{C}} \}$ is cofinal in $\naturals^{\naturals}$. 
Let $g \in \naturals^{\naturals}$ be arbitrary. Now consider the class $\Q=\mathcal{Q}_{\bar{\eta}, \bar{n}}$ as constructed in the previous section, where $\bar{\eta}(k) = \frac{8}{k}$ and $\bar{n}(k) = 8(g(k) +1 )$. Then according to Lemma~\ref{lemma:Qetanlearnable}, $\Q$ is learnable, as $ \lim_{k\to\infty}\bar{\eta}(k) =0$. 
Furthermore, we know that for every $k\in \naturals$ we have $P_{\frac{8}{k},\bar{n}(k)} \subset \Q$. Thus, by Lemma~\ref{lemma:lowerbound}, for every $k\in \naturals$: $g(k) < \frac{\bar{n}(k)}{4} \leq m_{Q} (\frac{1}{k},\frac{1}{7})$. Thus, $g \leq_{e.d} m_{Q}(\frac{1}{\cdot},\frac{1}{7}) $. Therefore $ \{ m_{\mathcal{Q}}(1/\cdot, 1/7) : \Q \in m_{\mathcal{C}} \}$ is indeed cofinal in $\naturals^{\naturals}$.
As the bounds of Lemma~\ref{lemma:Qetanlearnable} and Lemma~\ref{lemma:lowerbound} both hold for the realizable distribution learning, we can prove that  $ \{ m_{\mathcal{Q}}^{rlzb}(1/\cdot, 1/7) : \Q \in m_{\mathcal{C}} \}$ is cofinal in $\naturals^{\naturals}$ by the same construction and argument. Thus, we have proved our claim.

\end{proof}

We can now focus our attention the impossibility of qualitative characterization of distribution learning and finally prove Theorem~\ref{thm:distributionfinchar}.

\begin{proof}[Proof of Theorem~\ref{thm:distributionfinchar}]
    We only need to show that the two conditions from Theorem~ref{thm:finchar} are fulfilled by the problem of distribution learning. Condition 1. holds, as according to Theorem~\ref{thm:finiteclasses} every finite class of distributions is learnable. This means we can define a learner for any finite union of learnable sets, by running the learners for each of the learnable sets on the input to create a finite set of candidates and then learn the candidate set. Condition 2. holds by the following construction:
    $H_0=\{\delta_0\}$ and $H_{k} = \mathcal{Q}_{\bar{\eta}_n, \bar{n}}$, as defined in equation (1), where $\bar{\eta}_k(i) = 1/k$ and $n(i)=i$. It is clear that, $H_0$ is an $\epsilon_k$-approximation of $H_k$ for $\epsilon_k =1/k$ as all elements of $H_k$ have $\frac{k-1}{k}$ mass on the point $0$. Furthermore we have $\lim_{k\to\infty} \epsilon_k =0$. Lastly, we have to argue that ever $H_k$ is not learnable. We now note that for every $n\in \mathcal{N}$ the class $P_{\epsilon_k, 4n} \subset H_k$ (as defined in Section 3). Thus we can apply Lemma~\ref{lemma:lowerbound} to obtain, for every $n\in \mathbb{N}$, $m_{H_k}(\frac{\epsilon_k}{8},\frac{1}{7}) \geq n$.
    We note that any instance of the word "learnable" in this proof can either mean "realizale PAC learnable" or "3-agnostic PAC learnable". The proof is correct in both cases.
\end{proof}

\subsection{Polynomial complexity distribution learning}
Another, more restricted definition of learning, is one that requires specific bounds on the sample complexity.
\begin{definition}
    We say a class $H$ is polynomially PAC learnable, if $m_H(\epsilon,\delta) \in poly(1/\epsilon,1/\delta)$.
\end{definition}
We note, that for many learning tasks, like binary classification, classes are polynomially PAC learnable, if and only if they are PAC learnable. However, we have seen in this section that for the task of distribution learning, there are PAC learnable classes which are not polynomially PAC learnable. Arguably, in many scenarios one is more interested in polynomially learnable classes, as they cover all interesting classes, and anything that requires arbitrarily large sample complexity might be impractical for learning. We therefore pose the question, whether it is possible to give a \emph{qualitative} characterization of \emph{polynomially learnable} classes. In the case of distribution learning, we can give a partial answer, showing that there is no \emph{uniformly-bounded} finitary characterization of polynomial distribution learning. 

\begin{theorem}\label{thm:polydistrfinchar}
    There is no uniformly-bounded finitary characterization of polynomial distribution learning (w.r.t TV-distance). This result holds for both the realizable and the 3-agnostic case of distribution learning.
\end{theorem}

\begin{proof}
Assume, b.w.o.c. that $W$ is a uniformly bounded finitary characterization of polynomial distribution learning. From $W$ being uniformly bounded, we know that for every $\alpha$ there is $n$, such that for every class $H$, there is a subset $H_{\alpha}$ with $|H_{\alpha}|\leq n$ and such that for every $H'$ if $ H' \subset H_{\alpha}$, then $H'$ satisfies $\alpha$. 
 Let $W =\{\alpha_k: k \in \naturals\}$ an enumeration of $W$, that is ordered by the size of the corresponding $n_k$, i.e. such that for every $k \leq k'$, we have $n_k \leq n_{k'}$.
 We define $f(k) = k\cdot n_k$
 Consider $H_0 = \{\delta_0\}$ and $H_k = \Q_{\bar{\eta}_k,\bar{n}}$ with $\bar{n}(i) =i$ and $\bar{\eta}_k(i) = \max\{1/f(i), 1/f(k) \}$. Now $H_k$ is not learnable as for every $m\in \naturals$, $P_{1/f(k),4m} \subset H_k$, meaning that by Lemma~\ref{lemma:lowerbound} for 
 every $m$, $m_{H_k}(1/(8f(k)), 1/7) \geq m$. From the uniformly bounded finitary characterization, we know that there is $\hat{H}_k$, with $|\hat{H}_k| = n_k$ and
 every $H' $ if $H'\subset \hat{H}_k$ then $H'$ satisfies $\alpha_k$. Let
 $\hat{H} = H_0 \cup( \bigcup_{k\in \naturals} \hat{H}_k)$. By construction, we have that $\hat{H}$ satisfies $W$. 
 Furthermore, when aiming for $(\epsilon,\delta)$-success, it is sufficient to restrict our attention to learning the $\epsilon/4$ approximation $H_{\epsilon/4} =H_0 \cup(\bigcup_{k =1}^{\bar{\eta}^{-1}(\epsilon/4)} \hat{H}_k) $ of $\hat{H}$, where $\bar{\eta}^{-1}(\epsilon) = \min\{i: \text{ for all }j >i, \bar{\eta}(j)< \epsilon\} = \min\{k: k n_k \geq \frac{1}{\epsilon}\}$. Thus, $|H_{\frac{\epsilon}{4}}| \in poly(\frac{1}{\epsilon})$. From Theorem~\ref{thm:finiteclasses} we thus get that $H_{\epsilon/4}$ is polynomially learnable, which implies that $\hat{H}$ is polynomially PAC learnable w.r.t. to TV distance. Learning here can either mean 3-agnostic or realizable learnability. The result holds for both cases.
 
\end{proof}



\section{Impossibility of Characterizing Other Learning Tasks}\label{sec:others}
\subsection{Classification Learning for Distribution Classes}
We now define learning of distribution classes with respect to the 0/1-loss $L_P^{0/1}(h):= \mathbb{E}_{(x,y)\sim P} \indct{h(x)\leq y}$.
\begin{definition} \label{def:classificationlearningdistributions}
    A class of distributions $\mathcal{P}$ of distributions over $\mathcal{X}\times \{0,1\}$ is classification PAC-learnable, if there exist a learner $ \mathcal{A}$ and a sample complexity function $m_\mathcal{P}^{0/1}:(0,1)^2 \to \mathbb{N}$, such that for every $\epsilon,\delta >0$, every $P\in \mathcal{P}$ for every $m\geq m_{\mathcal{P}}(\epsilon,\delta)$ with probability $1-\delta$ over $S\sim P^m$,
    \[ L^{0/1}_P (\mathcal{A}(S) ) - L^{0/1}_P (f^{*}_P) < \epsilon,\]
    where $f_P^*$ denotes the Bayes classifier for the distribution $P$.
\end{definition}
We note that in the deterministic case, i.e. $L^{0/1}(f_P^*) =0$, this is equivalent to the learning problem of learning the class of all binary functions $2^{\X}$ with respect to some class of probability distributions $\mathcal{P}$. Since our constructions that show this learning problem cannot be characterized all fall into this deterministic case, they also show that the learning task proposed by Benedek-Itai(\cite{BenedekI91}) cannot be characterized, this section also resolves the open problems posed in \cite{BenedekI91,Dudley94}. 
Using our results from the previous sections, we can get the following uncharacterizability results for this learning task.
\begin{theorem}\label{thm:classweakdimension}
    There is no weak sample complexity dimension for classification learning of distribution classes (not even in the deterministic case).
\end{theorem}

\begin{theorem} \label{thm:classfinchar}
    There is no finitary characterization of  classification learning for distribution classes (not even in the deterministic case).
\end{theorem}
Theorem~\ref{thm:classweakdimension} follows from Lemma~\ref{lemma:cofinal} and Theorem~\ref{thm:classfinchar} follows as a corollary of Theorem~\ref{thm:finchar}.
Both theorems use a similar construction similar to the one used to show the results for distribution learning. We will now state this construction and the relevant upper and lower bounds for learning. For detailed proofs of the theorem, we refer the reader to the appendix.


Let 
\[P_{\eta,n}^{0/1} = \{(1-\eta)\delta_{(0,0)} + \frac{\eta |A|}{2n} U_{A\times \{0\}} + \frac{\eta |B|}{2n} U_{B\times\{1\}}: A\cup B= \{1,\dots 2n\} \text{ and } A\cap B = \emptyset \}.\]

Then for $\bar{\eta}:\naturals \to [0,1]$ and $\bar{n}:\naturals \to \naturals $ let 

\[\mathcal{Q}_{\bar{\eta},\bar{n}}^{0/1} = \bigcup_{i=1}^{\infty} P_{\bar{\eta}(i),\bar{n}(i)}^{0/1}.\] We will again use these kinds of constructions to show the theorems below.
We will now show that this class is learnable whenever $\bar{\eta}$ converges to $0$.
\begin{lemma}\label{lemma:classificationupperbound}
    If $\lim_{i\to\infty}\bar{\eta(i)} =0$, then $\Q = \Q_{\bar{\eta},\bar{n}}^{0/1}$ is classification PAC learnable.
\end{lemma}
The proof follows the same idea as the proof of Lemma~\ref{lemma:Qetanlearnable}. The proof can be found in the Appendix.

Furthermore, we can show a lower bound on the sample complexity for a given $P_{\eta,2n}$.
\begin{lemma}\label{lemma:classlowerbound}
    For $\Q=P_{\eta,2n}^{0/1}$, we have $m_{\Q}^{0/1}(\frac{\eta}{8},\frac{1}{7})\geq n$.
\end{lemma}
This Lemma follows directly from the proof of the no-free-lunch theorem in \cite{2shaibook}. For more detail we refer the reader to the appendix.

These two lemmas can now be used to show the theorems of this section.
\subsection{Learning Real-Valued Functions with Real-Valued Losses}
Let $\Z = \X \times \Y$.
We will now PAC learning of real-valued functions with continuous losses.
Let $\ell^g: 2^{\mathcal{X}} \times \mathcal{X} \times \mathcal{Y} \to \mathbb{R}_0^+$ be a (point-wise) loss, such there is a continuous function $g: \mathbb{R} \to \mathbb{R}$ with
\begin{itemize}
    \item for all $x\in \X,y\in Y, h\in 2^{\X}$:  $\ell^g(h,x,y) = g(|h(x)-y|)$.
    \item $g(0) =0$, i.e. perfect prediction incurs loss $0$.
    \item There is $a>0$, such that $g(a)>0$, i.e. some level of miss-estimation will incur positive loss.
\end{itemize}
We now analyse PAC-learnability of a class $\H \subset \mathcal{Y}^{\mathcal{X}}$ with respect to $L^{g}(h,P) = \mathbb{E}_{(x,y)\sim P}[\ell^g(h,x,y)]$.

\begin{definition}
    We say a class $\H\subset \Y^{\X}$ is 1-agnostic PAC-learnable w.r.t. $L^g$, if there exists a learner $\learner$ and a sample complexity function $m_{\H}:(0,1)^2 \to \naturals$, such that for every $\epsilon,\delta>0$ and every distribution $P$ over $\X\times \Y$, we have for every $m\geq m_{\H}(\epsilon,\delta)$,
    \[Pr_{S\sim P^m}[L_P^g(\learner(S)) \leq \inf_{h\in H} L_P^g(h) + \epsilon ] \leq 1- \delta.\]
    We say a class $\mathcal{H}$ is PAC-learnable w.r.t. $L^g$ in the deterministic case if it is learnable with respect to all distributions $P$ with $\inf_{h\in \H} L^g(h) =0$. The sample complexity in the realizable case will be denoted by $m_H^{rlzb}$.
\end{definition}

We now state the main theorems of this subsection.

\begin{theorem}\label{thm:contlossweakdim}
    There is no weak sample complexity dimension for PAC learning real-valued function classes with respect to $L^g$ (in neither the realizable nor the 1-agnostic case).
\end{theorem}
\begin{theorem}\label{thm:contlossfinchar}
    There is no finitary characterization of PAC learning real-valued function classes with respect to $L^g$ (in neither the realizable nor the 1-agnostic case).
\end{theorem}
We note that these results do not stand in contradition to the positive result on characterizing learnability for real-valued functions given by \cite{AlonBCH97}, as this result gives a characterization for $\epsilon$-weak learnability for every $\epsilon$, rather than a characterization for PAC-learnability.
To show our theorems, we need a similar construction as before, which we then use to apply Lemma~\ref{lemma:cofinal} and Theorem~\ref{thm:finchar} respectively.

We will now state the needed construction and then prove learnability as well as a lower bound on the sample complexity needed for the theorems.
Let $g_{max} =  \min \{ \max_{a>0}g(a), 1\} $ and $g^{-1}: [0,g_{max}] \to \mathbb{R}_0^+$.

For some $\eta\in [0,g_{max}]$,$n\in \naturals$ and $A\subset\{1,\dots,n\}$, let
\[f_{\eta,n}^A(x) = \begin{cases}
   g^{-1}(\eta) & \text{ if } x\in A\\
    0 &\text{otherwise}
\end{cases} \]
Now for a fixed $\eta$ and a fixed $n$, we define
\[ F_{\eta,n} = \{f_{\bar{\eta}(i), \bar{n}(i)}^A: A\subset \{1,\dots,\bar{n}(i)\}\}.\]

We then define $H_{\bar{\eta},\bar{n}}$ for sequences $\bar{\eta}:\naturals \to [0,g_{max}]$ and $\bar{n}:\naturals \to \naturals$, as

\[\H_{\bar{\eta},\bar{n}}= \bigcup_{i=1}^{\infty} \{f_{\bar{\eta}(i), \bar{n}(i)}^A: A\subset \{1,\dots,\bar{n}(i)\}\}.\]

\begin{lemma}\label{lemma:contlossupperbound}
    If $\lim_{i\to\infty}\bar{\eta(i)} =0$, then $\H_{\bar{\eta},\bar{n}}$ is classification is 1-agnostic PAC learnable with respect to $L^g$. 
\end{lemma}

\begin{lemma}\label{lemma:contlosslowerbound}
    For $F = F_{\eta,2n}^{0/1}$, we have $m_{F}(\frac{\eta}{8},\frac{1}{7})\geq m_{F}^{rlzb}(\frac{\eta}{8},\frac{1}{7})\geq n$.
\end{lemma}

Now following the same proof strategy as for our previous results, this construction can be used fulfill the requirements needed to prove the theorems of this section. We refer the reader to the appendix for the details of the proof.

\section{Discussion} \label{sec:discussion}


We showed the uncharacterizability of a variety of learning tasks,
for which the questions of coming up with such characterizations have  subject to research for many years. We discussed both the quantitative and quantitative characterizations and proposed some general properties of learning tasks that imply that their learnability is not captured by such characterizations. 

Our work was in some part inspired by the work of \cite{BD_undecidable} which was the first (and to our knowledge only one so far) to show the existence of a learning task that cannot be characterized.
While their work laid the groundwork and gave a first formal definition of general dimensions for statistical learnability, we extended those definitions and also proposed a definition for quantitative characterizability.

 Our work expands the understanding of uncharacterizability of learning problems in crucial ways. The results from \cite{BD_undecidable} 
 applied to a newly defined learning task - Expectation Maximization - and relied on the existence of classes whose learnability is undecidable in ZFC. In contrast, our results apply to several natural learning tasks whose characterizability have thus far eluded the community (\cite{Diakonikolas16}) and are `absolute' in not referring  to notions of provability. 

 

Another distinction to the EMX learning task in \cite{BD_undecidable} is the fact that the definition of EMX learning requires learning to be \textit{proper}, i.e. the output of a successful learner needs to be element of the class that is being learned. Without this requirement the EMX setting becomes trivial, as any class can be learned by the constant learner that outputs the whole domain set for any input.
In contrast, our results address the more general case of learning (and can also be easily extended to the proper case as well).

We note that all of our results rely on the construction of a sequence $\epsilon$-weakly learnable classes for decreasing $\epsilon$, which are not fully PAC learnable. For most tasks with known characterizations, there is an equivalence between weak learning and PAC learning. It might be interesting to further explore the connection between that equivalence and the characterizability of a learning problem. 

We believe that we have not exhausted the implications of our approach and that our definitions of characterizations and our techniques are also applicable to more learning tasks.

\subsection*{Acknowledgements}
We would like to thank Alex Bie and Ruth Urner for helpful discussions.
Tosca Lechner was supported by a Waterloo Apple PhD fellowship and a Vector Research Grant. Shai Ben-David has been supported as a Canada AI CIFAR Chair.






\appendix
\section{Proofs}

\begin{proof}[Proof of Lemma~\ref{lemma:lowerbound}]
Our proof follows a typical no-free-lunch-style argument.
    Consider $\mathcal{Q}'= \{(1-\eta)\delta_0 + \eta U_{A} : A\subset \{1,\dots,4n\} \text{ and } |A| = n\}$. We will show a lower bound of learning this class of distributions and then conclude that this lower bound also holds for $\mathcal{Q}$, as $\mathcal{Q}'\subset \mathcal{Q}$. 
    Now let $\learner$ be any learner. Furthermore, let $S_1,\dots S_k$ be the set of all sequences of size $n$ with elements in the set $\{0,\dots,4n\}$. 
    We have

       \[ \mathbb{E}_{S\sim q_i}[ TV(q_i,\learner(S))] = \sum_{j=1}^{k} q_{i}^n(S_j) TV(q_i, \learner(S_{j})) \]
      



Now for every $S_{j}$ and every $q_{i_1},q_{i_2}\in \mathcal{Q}'$, we have that
if $S_{j} \in \supp(q_{i_1}^n)$, then

\[
q_{i_2}^n(S_{j}) = \begin{cases} 
q_{i_1}^n(S_{j}) &\text{ if } S_{j} \subset supp(q_{i_2}^n) \\
0 & \text{otherwise}
\end{cases}
\]

Let us denote $ C_j \{x \in \{1,\dots,4n\}: x\in S_j\}$.
Furthermore for a set $A = \{x_1,\dots,x_p\}$ with $ C_j \subset A\subset \{1,\dots,4n\}$ and $|A| \leq 2n $, let $\bar{A}= \{x'_1,\dots, x'_{p'}\}$, be such that $x_{l_1}< x_{l_2}$ and $x'_{l_1} < x'_{l_2}$ for $l_1< l_2$. Now let us define $g_j(A) = C_j \cup \{x'_{l}\in \bar{A}: x_{p-l} \in A\setminus C_j\}$. We note that $ |g_j(A)| = |A|$, $A\cap g_j(A) = C_j$ and $g_j(g_j(A)) =A$.

Now for any $q_i\in \mathcal{Q}'$, let us denote $A_i = \supp(q_i)\setminus\{0\}$. Now let us define
\[
f_j(q_i) = \begin{cases} (1-\eta)\delta_{0} + \eta U_{g_j(A_i)} &   \text{ if }  C_j \in A_i \\
\delta_{4n+1} & \text{ otherwise. }
\end{cases}
\] 

We note that if $q_i^n(S_j)>0$, then $C_j\in A_i$ and $f_j(f_j(q_i)) = q_i$ and $f_j(q_i)\in \mathcal{Q}'$. In this case we furthermore have $TV(q_i,f_j(q_i)) \geq \frac{\eta}{2}$ Furthermore if $q_i^n(S_j)=0$, then $f_j(q_i)^n(S_j) = 0$. Taking both of these cases together we have for all $i$ and for all $j$: $q_i^n(S_j) =f_j(q_i)^n(S_j) $. 

Now we can put everything together into a no-free-lunch style argument.

\begin{align*}
 \max_{q_i\in \mathcal{Q}'}\mathbb{E}_{S\sim q_i}[ TV(q_i,\learner(S))] &= \max_{q_i\in \mathcal{Q}'}\sum_{j=1}^{k} q_{i}^n(S_j) TV(q_i, \learner(S_{j})) \\
 &\geq \frac{1}{T}\sum_{i=1}^T \sum_{j=1}^{k} q_{i}^n(S_j) TV(q_i, \learner(S_{j})) \\
 & =  \frac{1}{2T}\sum_{i=1}^T \sum_{j=1}^{k} q_{i}^n(S_j) TV(q_i, \learner(S_{j})) +\frac{1}{2T}\sum_{i=1}^T \sum_{j=1}^{k} f_j(q_{i})^n(S_j) TV(f_j(q_i), \learner(S_{j}))  \\
 &= \frac{1}{2T} \sum_{i=1}^T \sum_{j=1}^{k} q_{i}^n(S_j) TV(q_i, \learner(S_{j})) + f(q_{i})^n(S_j) TV(f_j(q_i),\learner(S_{j}))\\
 &= \frac{1}{2T} \sum_{i=1}^T \sum_{j=1}^{k} q_{i}^n(S_j) ( TV(q_i, \learner(S_{j})) + TV(f_j(q_i),\learner(S_{j})) )\\
 &\geq \frac{1}{2T} \sum_{i=1}^T \sum_{j=1}^{k} q_{i}^n(S_j)  TV(q_i, f_j(q_i)) \\
 &\geq \frac{1}{2T} \sum_{i=1}^T \sum_{j=1}^{k} q_{i}^n(S_j) \frac{\eta}{2} = \frac{\eta}{4}.
 \end{align*}


    




Now, by Lemma B.1 of \cite{2shaibook}, we get

\[ \max_{q_i\in \mathcal{Q}'}\mathbb{P}_{ S\sim q_i^n}[TV(q_i, \learner(S)) \geq \frac{\eta}{8}] = \max_{q_i\in \mathcal{Q}'}\mathbb{P}_{S\sim q_i^n}[TV(q_i, \learner(S)) \geq 1 - \frac{7\eta}{8}] \geq \max_{q_i\in \mathcal{Q}'}\frac{\mathbb{E}_{ S\sim q_i^n}[TV(q_i, \learner(S))]-\frac{1}{8}}{\frac{7}{8}} \geq \frac{1}{7} \]
Thus, $m_{\mathcal{Q}'}(\frac{\eta}{8},\frac{1}{7}) \geq n$. Therefore $m_{\mathcal{Q}}^{rlzb}(\frac{\eta}{8},\frac{1}{7}) \geq n$.
\end{proof}

\begin{proof}[Proof of Lemma~\ref{lemma:classificationupperbound}]
The proof is equivalent to the proof of Lemma~\ref{lemma:Qetanlearnable}. Let $\epsilon>0$. $\Q_{\frac{\epsilon}{2}}= \bigcup_{i: \bar{\eta}(i) > \frac{\epsilon}{2}} P_{\bar{\eta}(i), \bar{n}(i)} $. From $\lim_{i \to \infty} \bar{\eta}(i) = 0$, we know that this class is finite.  For a class $Q$, define the hypothesis class $H(Q)=\{ h\in \{0,1\}^{\X}: \exists q \in Q\text{ with } q(x,1)\geq q(x,0) \text{ if and only if } h(x)=1\}$. Now let us consider $H = H(\Q_{\frac{\epsilon}{2}})$. By construction, this class is finite and can therefore be PAC learned (in the binary classification sense). Furthermore, we have constructed $H$ in such a way that for every $q\in \Q$, we can bound $\inf_{h\in H} L_q^{0/1}(h) \leq \inf_{p\in \Q_{\frac{\epsilon}{2}}}(TV(p,q) + \inf_{h\in H} L_p^{0/1}(h)) \leq \frac{\epsilon}{2}$. Thus, if we have a learner that $(\epsilon/2,\delta)$-successfully learns $H$, we can use it to successfully $(\epsilon, \delta)$ learned $\Q$ w.r.t. to $L^{0/1}$. 
\end{proof}

\begin{proof}[Proof of Lemma~\ref{lemma:classlowerbound}]
Let us denote $U = U_{\{1,\dots,4n\}}$. We note that every mixture distribution in $ p = P_{\eta,2n}^{0,1}$, has the same marginal $p_{\X} = (1-\eta)\delta_0  + \eta U$. We furthermore note that we get all $2^{4n}$ possible labelings of $\{1,\to,4n\}$. We can thus write $P_{\eta,2n}^{0/1} =\{ (D,h): D = (1-\eta)\delta_0  + \eta U, h(0)=0, \exists h'\in \{0,1\}^{\{1,\to,4n\}}: \text{ for all } x\in \{1,\dots,4n\} h(x)=h'(x) \}$. Furthermore, for any $h\in 2^{\naturals}$, we can decompose the loss $L^{0/1}_{P}(h) = (1-\eta) \indct{h(x)\neq 0} + \eta L_{U}^{0/1}(h) \geq \eta  L_{U}^{0/1}(h)$. Now for every learner $\mathcal{A}$, we can derive the lower bound $\max_{P\in P_{\eta,2n}} \mathbb{E}_{S\sim P} [L_{U,h}^{0/1}(A(S))] \geq \frac{1}{4}$, according to the same argument as in the no-free-lunch theorem in \cite{2shaibook}. Thus we have $\max_{P\in P_{\eta,2n}} \mathbb{E}_{S\sim P} [L_{P}^{0/1}(A(S))] \geq \frac{\eta}{4}$. Thus by Lemma~B.1 of \cite{2shaibook}, we get the claimed result.

\end{proof}

\begin{proof}[Proof of Theorem~\ref{thm:classweakdimension}]
   Let $\C$ be the collection of all classification-learnable distribution classes. According to Lemma~\ref{lemma:cofinal} it is sufficient to show that $\{m_{Q}(\frac{1}{\cdot},\frac{1}{7}): \Q \in \C\}$ is cofinal in $\naturals^{\naturals}$. Let $g\in \naturals^{\naturals}$ be arbitrary. Now consider the class $\Q =\Q_{\bar{\eta},\bar{n}}$ with $\bar{\eta}(k) = \frac{8}{k}$ and $\bar{n}(k) = 8( g(k) + 1)$. Then, according to Lemma~\ref{lemma:classificationupperbound} is learnable. Furthermore, since for every $k\in \naturals$, we have $P_{\frac{8}{k}, \bar{n}(k)}\subset \Q$, by Lemma~\ref{lemma:classlowerbound} we have $g(k) \leq \frac{\bar{n}(k)}{4} < m_{Q}(1/k,1/7)$. This shows that $\{m_{Q}(\frac{1}{\cdot},\frac{1}{7}):\Q\in\C\}$ is indeed cofinal in $\naturals^{\naturals}$, which concludes the proof of our claim. 
\end{proof}

\begin{proof}[Proof of Theorem~\ref{thm:classfinchar}]
We will again use Theorem~\ref{thm:finchar} to show this claim. Thus we only need to show the conditions for Theorem~\ref{thm:finchar} hold.
\begin{itemize}
    \item \emph{Any union of finitely many learnable classes is learnable}: Let us denote this class by $Q = \bigcup_{i=1}Q^i$. For every $Q^i$ there is a learner $A_i$ with sample complexity $m_i$. Let $\epsilon>0,\delta>0$ and let $m_{max} = \max(\{m_i(\epsilon/2,\delta/2): 1\leq i \leq k\} \cup\{\frac{4(\log(k) +\frac{2}{\delta})}{\epsilon^2}\} )$. Now for some $q\in Q$, let $S \sim q^{m_{max}}$. For every $i\in \{1,\dots,k\}$, we run $A_i$ on $S$ and denote the outputed hypthesis by $h_i$. We know that there is $j\in \{1,\dots,k\}$, such that $q\in Q^j$. By the success-guarantee of $A_j$, we know that with probability $1-\frac{\delta}{2}$, $L^{0/1}_{q}(h_j) \leq \frac{\epsilon}{2} + L_{q}^{0/1}(f_q)$. We can now use a PAC-learner for the finite class $H = \{h_i:1 \leq i\leq k\}$ (which we know to exist from PAC learnability of binary classification of hypothesis classes (see \cite{2shaibook}). We know that a sample complexity of $m_{max}$ guarantees a $(\epsilon/2,\delta/2)$-learning success for learning $H$. Taking everything together, we have constructed a learner that guarantees $(\epsilon,\delta)$-success for learning $Q$.
    \item We let $H_0=\{h_0: \text{ for all}  x\in \naturals h_0(x)=0\}$. Furthermore, we let $H_k$ be $H_k =\Q_{\bar{\eta}_k,\bar{n}}^{0/1}$, where $\bar{\eta}_k(i) = 1/k$ and $n(i)=i$ for all $i\in \naturals$. By Lemma~\ref{lemma:classlowerbound} all $H_k$ are not learnable. Furthermore, by construction, we have that for every $\epsilon_k=\frac{1}{k}$, $H_0$ is an $\epsilon_k$ approximation for $H_k$. Since $\lim_{k \to \infty} \epsilon_k =0$, the second condition of Theorem~\ref{thm:finchar} is fulfilled. This concludes our proof.
\end{itemize}

\end{proof}

\begin{proof}[Proof of Lemma~\ref{lemma:contlossupperbound}]
Let $h_0$ be the all-zero function, i.e. $h_0(x) =0$ for all $x\in \naturals$. We start by noting that for every $\eta$, every $n$ and every $A\subset \{1,\dots,n\}$ if $P = (D,f_{\eta,n}^A)$ for any marginal $D$ over $\naturals$, then $L_P^g(h_0) \leq g( g^{-1}(\eta)) = \eta$.
Furthermore, we know that finite hypothesis classes of hypotheses with finite range are learnable due to uniform convergence (which we get from first using Hoeffding on each of the elements of the class and then using a union bound).
We can now use the same proof idea as in Lemma~\ref{lemma:Qetanlearnable}. Let $\epsilon>0$. Since $\lim_{i\to\infty} \bar{\eta}(i) =0$, there is $N$, such that for every $M\geq N$, $\bar{\eta}(M) \leq \frac{\epsilon}{2}$. Thus we know that for every $P$ and every $h\in \F_{\bar{\eta},\bar{n}}$, there is $h'\in \bigcup_{i=1}^{N} \{f_{\bar{\eta}(i),\bar{n}(i)}^A: A\subset \{1,\dots,n\}\}$, such that $L_P^g(h') \leq L_P^g(h) +\frac{\epsilon}{2}$. We now use the fact that $\H' = \bigcup_{i=1}^{N} \{f_{\bar{\eta}(i),\bar{n}(i)}^A: A\subset \{1,\dots,n\}\}$ is a finite class of hypotheses with finite range and can therefore be successfully PAC-learned w.r.t. $L^g$. Now we can use the learner for $\H'$ with sample complexity $m_{\H'}$ on an i.i.d. sample of size $m\geq m_{H'}(\epsilon/2,\delta)$, to guarantee $(\epsilon,\delta)$-success for learning $\H_{\bar{\eta},\bar{n}}$. This 
\end{proof}

\begin{proof}[Proof of Lemma~\ref{lemma:contlosslowerbound}]
For this lower bound, let us only consider the realizable case. The agnostic case follows directly from it. 
We note that for a fixed $\eta$ and $n$, for every $h_1,h_2\in F_{\eta,2n}$, every $x\in \naturals$, we have either $\ell^g(h_1,x,h_2(x)) = \eta$ or $\ell^g(h_1,x,h_2(x)) = 0$. We can thus treat $l^g$ as a binary loss. We see that the statement now becomes equivalent to the No-Free-Lunch Theorem for binary classification (see Theorem 5.1 in \cite{2shaibook}) with the $\epsilon$-parameter in the sample complexity being scaled by $\eta$. Thus we can conclude $m_F^{rlzb}(\frac{\eta}{8},\frac{1}{7}) \geq n$. This concludes our proof.

\end{proof}

\begin{proof}[Proof of Theorem~\ref{thm:contlossweakdim}]
Let $\C$ be the class of all PAC-learnable function classes with respect to $L^g$. We know from Lemma~\ref{lemma:cofinal} that it is sufficient to show that $\{m_{H}(1/\cdot,1/7): H \in \C \}$ is cofinal in $\naturals^{\naturals}$. Now let $g\in \naturals^{\naturals}$ be arbitrary. We now construct a class $H$, such that $m_H(\frac{1}{\cdot},\frac{1}{7})$ eventually dominates $g$. Let $H = H_{\bar{\eta},\bar{n}}$, with $\bar{\eta}(k)=8/k$ and $\bar{n} = 4(g(k)+ 1)$ for every $k\in \naturals$. From Lemma~\ref{lemma:classificationupperbound} we know that the class is learnable as $\lim_{k \to \infty} \frac{8}{k} =0$. From Lemma~\ref{lemma:contlosslowerbound} we get $g(k)< \frac{\bar{n}(k)}{4} < m_H(1/k,1/7)$. We note that the notion of "learnability" used here can either refer to the realizable case and sample complexity function $m_H^{rlzb}$ or the 1-agnostic case and sample complexity function $m_H$. In either case, the statements are true, giving us both claims. 
\end{proof}

\begin{proof}[Proof of Theorem~\ref{thm:contlossfinchar}]
We use Theorem~\ref{thm:finchar} to prove the claim. We thus only need to show that the two conditions of the theorem are fulfilled. We will focus on the sub-task of learning function classes with range bounded by 1. Since we only use function classes of the form $H_{\bar{\eta}, \bar{n}}$ and those all consist of functions with range in $[0,1]$, this does not cause any issue. 
\begin{itemize}
    \item \emph{Every finite unition of learnable classes of hypotheses is learnable}:Let $\H = bigcup_{i=1}^k H^i$ be a union of learnable classes. Let $A_i$ denote the learner and $m_i$ denote the sample complexity for learning $H_i$. Now define $m_{max}(\epsilon,\delta) = \max(\{m_i(\epsilon/2,\delta/2)\}\cup \{\frac{4k + 4\log{2/\delta}}{\epsilon^2}\})$. Let $A$ be the learner that first runs every $A_i$ on an input sample to create a finite hypothesis class $H'(S) =\{A_i(S): 1 \leq i \leq k \}$ of candidates of size $k$ and then runs ERM on the finite hypothesis class. Now $P$ be some distribution over $\X\times \Y$ and let $m\geq m_{max}(\epsilon,\delta)$. Furthermore let $S\sim P^{m}$. We know that there is $j$, such that $\inf_{h\in H^j} L_P^g(h) = \inf_{h\in \H} L_P^g(h)$. Now $A_j$ guarantess that $\inf_{h\in H'(S)}L_P^g \leq \frac{\epsilon}{2}$. Furthermore, from Hoeffding's inequality and union bound, we get that a sample size of $m_{max}(\epsilon,\delta)$ is sufficient for an ERM to guarantee $(\frac{\epsilon}{2},\frac{\delta}{2})$-success, when learning any finite hypothesis class of size $k$ with functions with range $[0,1]$. Thus the learner $A$ successfully PAC learns $\H$ with sample complexity $m_{max}$. 
    \item We define $H_0=\{h_0\}$, with $h_0(x)=0$ for all $x\in \naturals$. Furthermore we let $H_{k} = H_{\bar{\eta}_k,\bar{n}}$, where $\bar{\eta}_k(i) = 1/k$ and $\bar{n}(i) =i$. From Lemma~\ref{lemma:contlosslowerbound} we know that none of the classes $H_k$ are learnable. Furthermore for every $k\in \naturals$, $H_0$ is an $\epsilon_k= \frac{1}{k}$ approximation of $H_k$, as argued in the proof of Lemma~\ref{lemma:contlossupperbound}. Lastly $\lim_{k\to \infty} \epsilon_k =0$. This concludes our proof.
\end{itemize}

\end{proof}

\bibliographystyle{unsrtnat}
\bibliography{references}

\end{document}